\documentclass[letterpaper, 10 pt, conference]{ieeeconf}  

\IEEEoverridecommandlockouts                              

\usepackage{color}

\usepackage{times}
\usepackage{epsfig}
\usepackage{graphicx}

\usepackage{amsmath}
\usepackage{amssymb}
\usepackage{bbm}

\usepackage[ruled,vlined,linesnumbered]{algorithm2e}
\usepackage{graphicx} 
\usepackage{subfigure}
\usepackage{subfloat}
\usepackage{verbatim}
\usepackage{booktabs}
\usepackage{multirow}
\usepackage{makecell}

\usepackage{hyphenat} 

\usepackage{colortbl,booktabs}
\usepackage{lscape}
\usepackage{url}
\usepackage{diagbox}

\usepackage{makecell}
\usepackage{xcolor,colortbl}

\title{\LARGE \bf
Simultaneous Feature and Body-Part Learning for Real-Time \\
Robot Awareness of Human Behaviors
}

\author{Fei Han$^{1}$, Xue Yang$^{1}$, Christopher Reardon$^{2}$, Yu Zhang$^{3}$, and Hao Zhang$^{1}$
\thanks{$^1$Fei Han, Xue Yang, and Hao Zhang are with  the Department of Computer Science, Colorado School of Mines, Golden, CO 80401, USA. {\tt\small  fhan@mines.edu, edyxueyx@gmail.com, hzhang@mines.edu}.
}
\thanks{$^2$Christopher Reardon is with the US Army Research Laboratory, Adelphi, MD 20783, USA.\newline
{\tt\small  christopher.m.reardon3.civ@mail.mil}.}
\thanks{$^3$Yu Zhang is with the Department of Computer Science and Engineering, Arizona State University, Tempe, AZ 85281, USA.
{\tt\small  yzhan442@asu.edu}.}}

\begin{document}

\newtheorem{definition}{Definition}
\newtheorem{theorem}{Theorem}
\newtheorem{lemma}{Lemma}
\newtheorem{proposition}{Proposition}
\newtheorem{property}{Property}
\newtheorem{observation}{Observation}
\newtheorem{corollary}{Corollary}

\maketitle
\thispagestyle{empty}
\pagestyle{empty}

\begin{abstract}
Robot awareness of human actions is an essential research problem in robotics with many important real-world applications,
including human-robot collaboration and teaming.
Over the past few years, depth sensors have become a standard device widely used by intelligent robots for 3D perception,
which can also offer human skeletal data in 3D space.
Several methods based on skeletal data were designed to enable robot awareness of human actions
with satisfactory accuracy.
However,
previous methods treated all body parts and features equally important, without the capability to identify discriminative body parts and features.
In this paper,
we propose a novel simultaneous \textit{Feature And Body-part Learning (\textbf{FABL})} approach that simultaneously identifies discriminative body parts and features,
and efficiently integrates all available information together to enable \emph{real-time} robot awareness of human behaviors.
We formulate FABL as a regression-like optimization problem
with structured sparsity-inducing norms to model interrelationships of body parts and features.
We also develop an optimization algorithm to solve the formulated problem, which possesses a theoretical guarantee to find the optimal solution.
To evaluate FABL,
three experiments were performed
using public benchmark datasets, including the MSR Action3D and CAD-60 datasets,
as well as a Baxter robot in practical assistive living applications.
Experimental results show that our FABL approach obtains
a high recognition accuracy with a processing speed of the order-of-magnitude of $10^4$ Hz,
which makes FABL a promising method to enable real-time robot awareness of human behaviors in practical robotics applications.
\end{abstract}

\section{Introduction}
In a wide variety of human-centered robotics applications, including human-robot teaming, human-robot collaboration, and robot-assisted living,
robot awareness of human actions (or behaviors) is essential 
for intelligent robots to understand humans,
make situationally appropriate decisions, and interact with and assist people.
However, robot awareness of human behaviors in real-world environments
is a challenging problem caused by significant variations of human motion,
diversity of human appearance,
and vision difficulties, including illumination variations and occlusion.
When implemented on robots,
additional challenges are encountered,
such as uncertainty in movement 
and dynamic backgrounds; 
Most importantly, the requirement of real-time performance demands timely robot planning and decision making.

Although human action understanding has been researched
in robotics and computer vision communities,
most previous techniques are based on local spatio-temporal visual features \cite{zhang2014simplex,wang2014learning},
which are generally incapable of dealing with the challenges
introduced by robotics applications (e.g., real-time performance).
With the emergence of affordable
structured-light or time-of-flight depth sensing technologies,
color-depth cameras have generally become a standard 3D visual sensing device for modern indoor robots.
The skeletal data of humans acquired from such sensors, as shown in Fig. \ref{fig:sub:skeleton},
provides the possibility to achieve real-time robot awareness of human behaviors \cite{shotton2011real},
which also provides benefits in comparison to local features,
including the invariance to viewpoint, human body scale and motion speed \cite{han2016space,Zhang15}.

\begin{figure}[tb]
  \subfigure[Skeletal data]{
    \label{fig:sub:skeleton} 
    \begin{minipage}[b]{0.165\textwidth}
      \centering
        \includegraphics[height=1.2in]{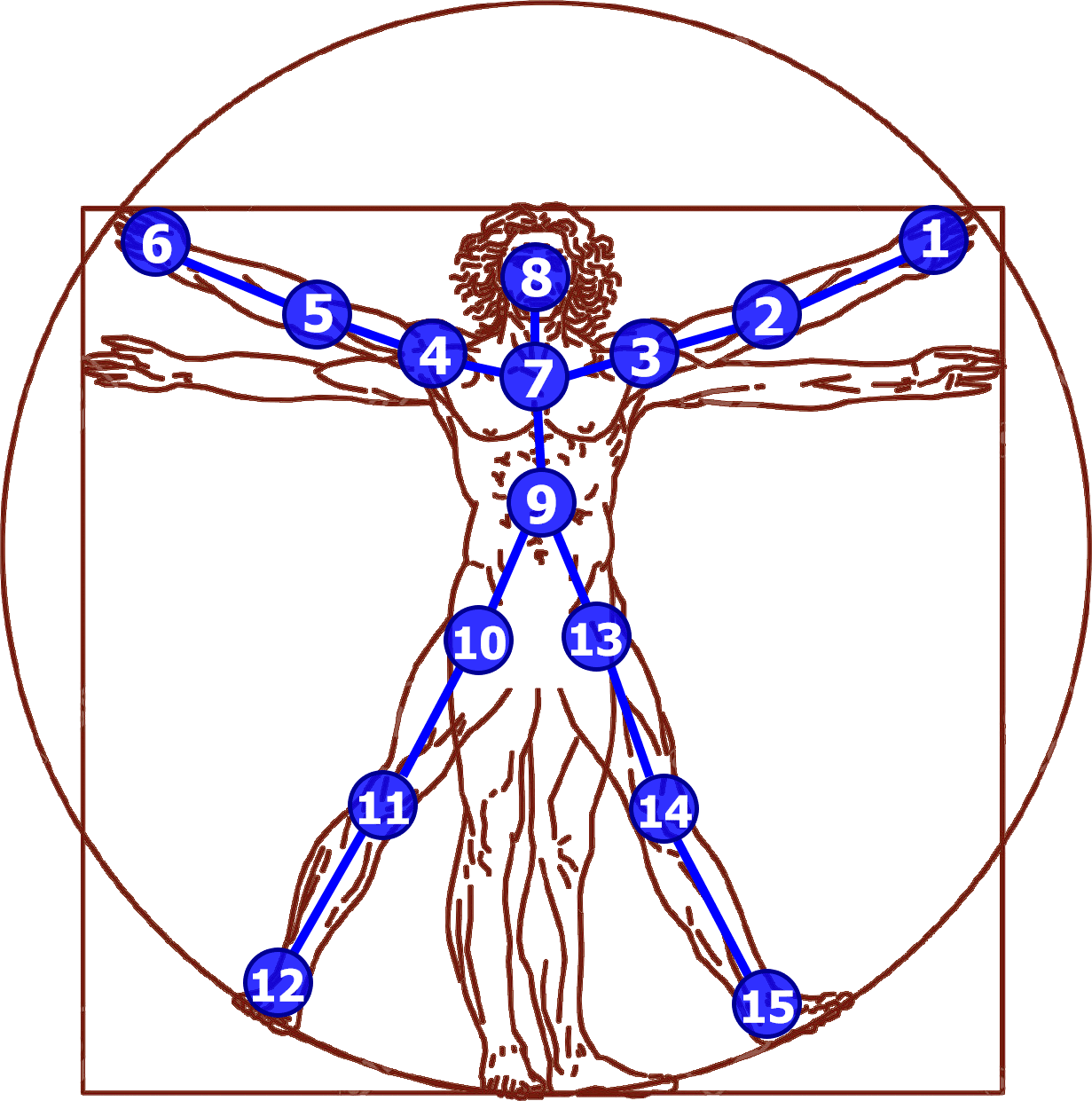}
    \end{minipage}}
   \subfigure[Discriminative features and body parts]{
    \label{fig:sub:FASJL} 
    \begin{minipage}[b]{0.29\textwidth}
      \centering
        \includegraphics[height=1.2in]{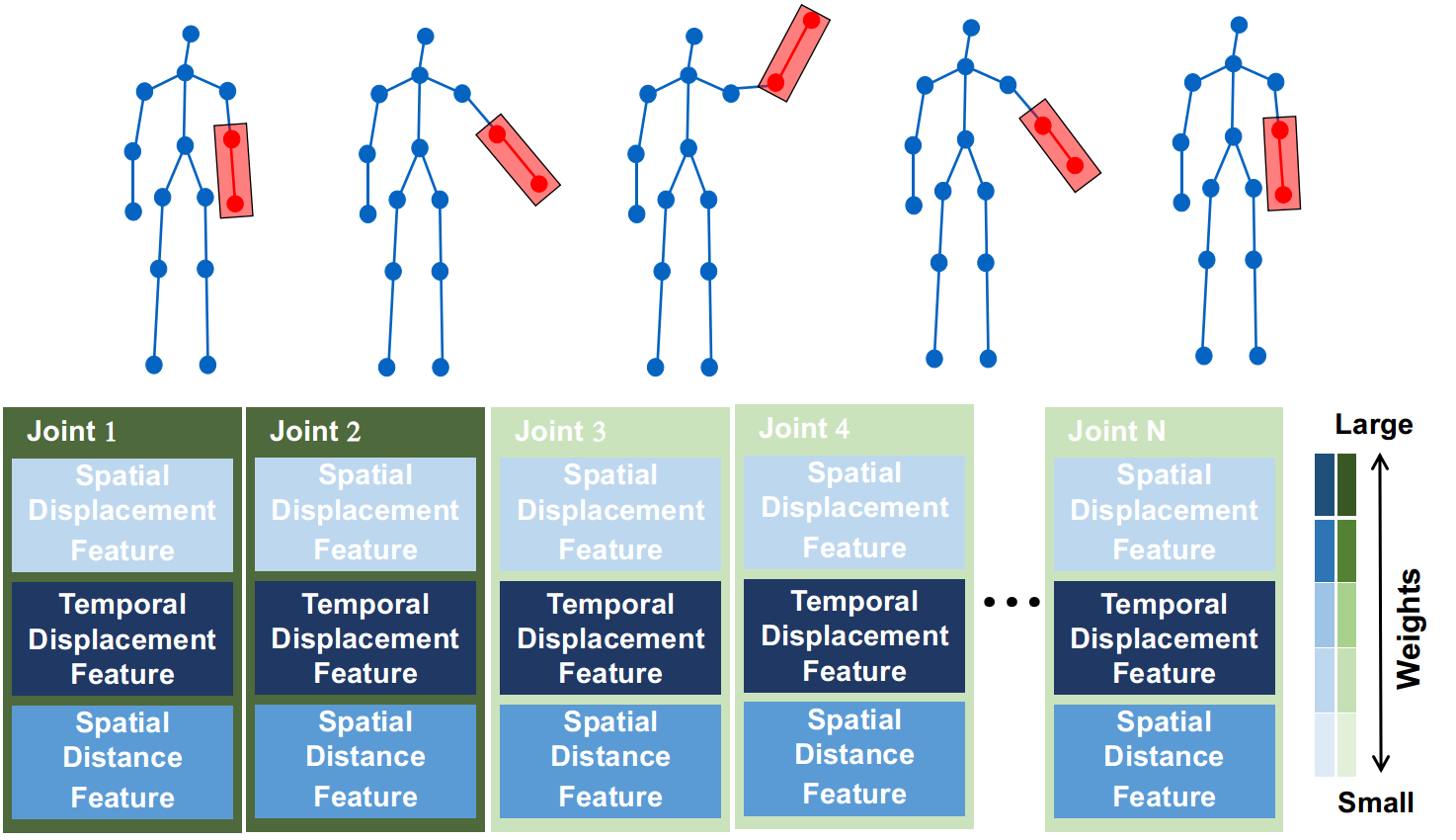}
    \end{minipage}}
  \caption{
A motivating example of the FABL approach,
which simultaneously learns discriminative skeleton joints
and multimodal heterogeneous features to enable real-time robot awareness of human behaviors.
}  \label{fig:treeStrucutre} 
\end{figure}

Because of these advantages,
skeleton-based action understanding methods have attracted increasing attention,
and many skeletal features and representations have been implemented during the last few years
, see \cite{han2016space} and references therein, i.e. joint rotation matrix \cite{Sung_ICRA12}, BIPOD \cite{Zhang15}, etc.
However, most existing methods apply only one type of skeletal feature \cite{Yang_CVPRW12,Sung_ICRA12},
while others simply concatenate several types of skeletal features
together into a single bigger vector to encode human actions \cite{gowayyed2013histogram,yang2014effective}.
The problem of autonomously learning the importance of skeletal features
and optimally integrating the multimodal features (different human activity representations extracted from skeletal data) together
has not yet been well addressed for real-time robot awareness of human behaviors.
Recently, methods based on \emph{body parts} (\emph{represented as joints in skeletal data})
instead of using complete skeleton data were studied to improve action recognition accuracy \cite{Zhang15,wang2012mining,ofli2014sequence}.
To remove irrelevant joints for specific behaviors, these methods use a subset of or select skeletal joints.
Although these methods obtained promising accuracy,
the selection is manual based upon fixed criteria and is not robust to various scenarios.
Furthermore, the question of how to integrate multimodal skeletal features into body-part methods has not been well answered.

In this paper, we introduce a novel \emph{Feature And Body-part Learning} (FABL) method to enable real-time robot awareness of human behaviors,
through learning discriminative skeletal features and body parts
simultaneously in the same optimization framework.
For learning the importance of body parts,
our approach is inspired by the insight that
typically a subset of body parts are more discriminative to recognize an action.
For example,
as demonstrated in Fig. \ref{fig:sub:FASJL},
only the waving arm and hand are important for the action of ``hand waving.''
Our FABL method is able to select discriminative body parts automatically for different behaviors.
Simultaneously,
 FABL learns the importance of heterogeneous skeletal features,
and integrates multimodal features to build a more discriminative representation to enable robot awareness of human behaviors.
Classification is seamlessly integrated in the FABL approach
 (i.e., no external classifier is required),
which further increases processing efficiency,
resulting in high-speed performance that is suitable for applications with real-time requirements.

The contributions of this paper are twofold:
\begin{itemize}
\item We propose a novel formulation and the FABL approach to perform simultaneous learning of discriminative body parts and skeletal features for real-time robot awareness of human behaviors.
\item We develop a new optimization algorithm to efficiently solve the formulated robot learning problem, which has a theoretical guarantee to converge to the global optimal solution.
\end{itemize}
We make the code that implements our FABL approach available at:
\url{http://hcr.mines.edu/code/FABL.html}.

The remainder of this paper is structured as follows.
Related work is described in Section \ref{sec:RelatedWork}.
Then, our FABL approach is detailed in Sections \ref{sec:Learning} and \ref{sec:Optimization}.
Experimental results are presented in Section \ref{sec:Experiments}.
After discussing several attributes of the proposed FABL method in Section \ref{sec:discussion},
we conclude this paper in Section \ref{sec:Conclusion}.

\section{Related Work}\label{sec:RelatedWork}
In this section, we conduct a review of techniques to understand human actions using skeletal data, including both complete skeletal data and partial body parts.

\subsection{Behavior Understanding Based on Skeletal Data}

Methods using 3D skeletal data to identify human actions
attracted increasing attention after the release of the affordable structured-light 3D sensing technology \cite{han2016space}.
A widely applied representation for human action understanding is based on skeletal joint displacements.
Chen and Koskela \cite{chen2013online} implemented a feature extraction
method based on pairwise relative position of skeletal joints
with normalization, and actions were classified by multiple extreme learning machines.
Wei \emph{et al.} \cite{Wei2013} implemented a hierarchical graph to represent spatio-temporal joint positions and displacements,
where the differences in skeletal joint positions between two successive frames were defined as features.
Besides joint displacements, many methods based on joint orientations were also implemented.
Sung \emph{et al.} \cite{Sung_ICRA12} computed the orientation matrix of each joint
with respect to the camera,
then transformed the matrix to obtain this joint
orientation with respect to the human torso,
showing their representation was invariant to the sensor's location.
Another popular category of skeleton-based methods directly
use raw joint position information for human action understanding.
Wei \emph{et al.} \cite{wei2013concurrent} developed wavelet features to represent a sequence of 3D skeletal joints,
and a concurrent action detection model to understand human behaviors.

Most of the previous skeleton-based methods utilized only one category of skeleton-based features.
Several recent studies indicate that recognition accuracy can be improved by combining multiple skeletal features together.
A feature construction approach was introduced in \cite{Yang_CVPRW12}
that concatenates static posture, movement, and offset values into a single bigger feature vector,
and utilizes a naive Bayes classifier to perform multi-class action classification.
Yu \emph{et al.} \cite{yu2015discriminative} used three categories of skeletal features,
including pairwise joint distance, spatial joint coordinate, and temporal variation of joint locations,
to construct a mixed representation.
A similar skeleton-based representation was implemented by \cite{Masood2011}, incorporating pairwise joint distances and temporal joint location changes together.
However, most previous techniques
 simply concatenated different categories of features
without considering the importance of each skeletal feature category.
The research problem of how to autonomously learn and fuse heterogeneous skeletal features for real-time robot awareness of human actions has not yet been well studied.

The proposed FABL approach addresses this problem by integrating heterogeneous multimodal skeletal features
through learning the importance of each feature category,
along with learning discriminative body parts,
to accurately interpret human actions.

\subsection{Representation Based on Body Parts}

Skeletal human representations based on
body part models have been widely studied in the past few years.
Because these mid-level body part models can partially take into account
the physical structure of human body,
they can yield improved discrimination power to represent humans \cite{Zhang15}.

Wang \emph{et al.} \cite{wang2013approach} implemented a method that
decomposed a body model into five parts,
including left/right arms/legs and the torso, each consisting of a set of joints,
to represent human behaviors in space and time dimensions.
A spatial-temporal And-Or graph model was implemented in \cite{xiaohan2015joint}
to represent humans at three levels including poses, spatiotemporal-parts, and parts.
The hierarchical human body structure captures the geometric and appearance variation of humans at each frame.
A deep neural network was introduced in \cite{du2015hierarchical} to create a body part model
and the correlation of body parts was investigated,
which can automatically obtain mid-level features
that were more descriptive than low-level features extracted from individual human skeleton joints.
Several methods were also proposed to select more descriptive human body joints
\cite{wang2014learning,wang2012mining,ofli2014sequence,chaaraoui2014evolutionary,reyes2011featureweighting,patsadu2012human,
huang2014action}.

Bio-inspired body part models are also commonly applied
to extract mid-level features for skeleton-based representation construction,
which are typically based on body kinematics or human anatomy.
Chaudhry \emph{et al.} \cite{Chaudhry2013} implemented  bio-inspired mid-level features
to represent human activities based on 3D skeleton data,
by leveraging the findings in the research area of static shape encoding in the primate cortex's neural pathway.
By showing different 3D shapes to primates
and measuring their neural responses,
the primates'
internal shape representation was estimated,
which was then used to extract body parts to create skeleton-based representations.
Zhang and Parker \cite{Zhang15}
proposed a new bio-inspired predictive orientation decomposition representation,
which was inspired by the biological research in human anatomy.
This approach decomposed a body model into five body parts,
and projected 3D human skeleton trajectories onto three anatomical planes.
Through estimating future skeleton trajectories,
this method is able to predict future human motions.

Despite the promising results obtained by the methods based on body parts,
which mutually partition the body model into several body parts
or select a set of skeletal joints according to predefined criteria,
previous techniques did not model the discrimination difference of human joints but simply include or exclude certain joints.
In this paper,
we introduce a new approach to automatically learn discriminative skeletal joints
without predefined manual selection criteria.

\section{The FABL Approach} \label{sec:Learning}

In this section, we describe our FABL method
that simultaneously learns discriminative skeletal features and body parts
to enable real-time robot awareness of human behaviors.

\textit{Notation.}
In this paper, we denote matrices using boldface capital letters,
and vectors using boldface lowercase letters.
We represent the $\ell_1$-norm of a vector $\mathbf{v} \in \Re^n$ using $\|\mathbf{v}\|_1 = \sum_{i=1}^{n} |v_i|$,
and the $\ell_2$-norm of $\mathbf{v}$ as $\|\mathbf{v}\|_2 = \sqrt{\mathbf{v}^\top\mathbf{v}}$.
Given a matrix $\mathbf{M} \!=\! \{m_{ij}\}\in \Re^{m\times n}$,
we refer to its $i$-th row as $\mathbf{m}^{i}$ and the $j$-th column as $\mathbf{m}_{j}$.
We denote the Frobenius norm of the matrix $\mathbf{M}$ as $\|\mathbf{M}\|_F=\sqrt{\sum_{i=1}^m\sum_{j=1}^n{m_{ij}^2}}$.

\subsection{Problem Formulation}

Given a collection of $n$ data instances,
the skeletal matrix is denoted as $\mathbf{X}=\left[\mathbf{x}_1,\cdots,\mathbf{x}_n\right]\in\Re^{d\times n}$,
where $\mathbf{x}_i\in\Re^{d}$ is the vector of all skeletal features for the $i$-th data instance.
When heterogeneous skeletal features are used,
each vector $\mathbf{x}_i\in\Re^{d}$ consists of $m$ modalities
such that $d=\sum_{j=1}^{m}{d_j}$.
Within each modality, the skeletal features  are further divided into
$s$ partitions, and each partition contains features from a skeleton joint.
Then, we formulate robot awareness of human behaviors as a problem of
dividing $\left\{\mathbf{x}_i\right\}_{i=1}^n$ into $c$ behavior categories through exploiting all available information
from heterogeneous feature modalities and skeleton joints,
using a regression-like classification objective as follows:
\begin{eqnarray}\label{eq:objective_F_norm}
\mathop{\text{min}}\limits_{\mathbf{W}}{\|\mathbf{X}^\top\mathbf{W}+\mathbf{1}_n\mathbf{b}^\top-\mathbf{Y}\|_F^2},
\end{eqnarray}
where $\mathbf{1}_n \in \Re^{n\times 1}$ is the constant vector of all $1$'s, $\mathbf{b}\in\Re^{c\times 1}$ is the intercept vector,
 $\mathbf{Y}=\left[\mathbf{y}_1,\cdots,\mathbf{y}_n\right]^\top\in\Re^{n\times c}$ denotes the behavior category indicator matrix, and
$\mathbf{y}_i\in\Re^c$ denotes the category indicator vector for the feature vector $\mathbf{x}_i$
with $y_{ij}$ indicating how likely $\mathbf{x}_i$ belongs to the $j$-th category.
The label matrix $\mathbf{Y}$ of the data instances is given in the training phase.
Then, the value of $\mathbf{b}$ in Eq. (\ref{eq:objective_F_norm}) can be calculated by $\mathbf{b}=\mathbf{Y}^\top\mathbf{1}_n/n$.

The solution of the optimization problem in Eq. (\ref{eq:objective_F_norm}) is the
parameter matrix $\mathbf{W}  = [\mathbf{w}_1, \mathbf{w}_2, \dots, \mathbf{w}_c] \in\Re^{d\times c}$,
which contains the weights $\mathbf{w}_i \in\Re^{d}$ of each feature modality and skeletal joint with respect to the $i$-th behavior category.
The parameter matrix $\mathbf{W}$ is denoted as:
\begin{eqnarray}
\mathbf{W} = \left[
  \begin{array}{ccc}
    \mathbf{w}_1^1 & \cdots & \mathbf{w}_c^1\\
    \vdots & \ddots & \vdots \\
    \mathbf{w}_1^m & \cdots & \mathbf{w}_c^m\\
  \end{array}
  \right],
\end{eqnarray}
where $\mathbf{w}_p^q\in\Re^{d_q}$ indicates the weights of the $q$-th modality including all skeleton joints with respect to the $p$-th behavior category, which is denoted as
$\mathbf{w}_p^q = \left[\mathbf{w}_p^{q^1};\mathbf{w}_p^{q^2};\cdots;\mathbf{w}_p^{q^s}\right]$,
and $\mathbf{w}_p^{q^r}\in\Re^{d_{q^r}}$ represents the weights of the $r$-th skeleton joint within the $q$-th modality with respect to the $p$-th human behavior category,
where $d_{q^r}$ is the dimension of features that are obtained from the $r$-th skeleton joint in the $q$-th modality,
satisfying $\sum_{r=1}^{s}{d_{q^r}}=d_q$, and
$s$ is the number of skeleton joints in each modality.
An illustration of the weight matrix is presented in Fig. \ref{eq:W}.

\begin{figure}[ht]
\centering
\includegraphics[width=3.5in]{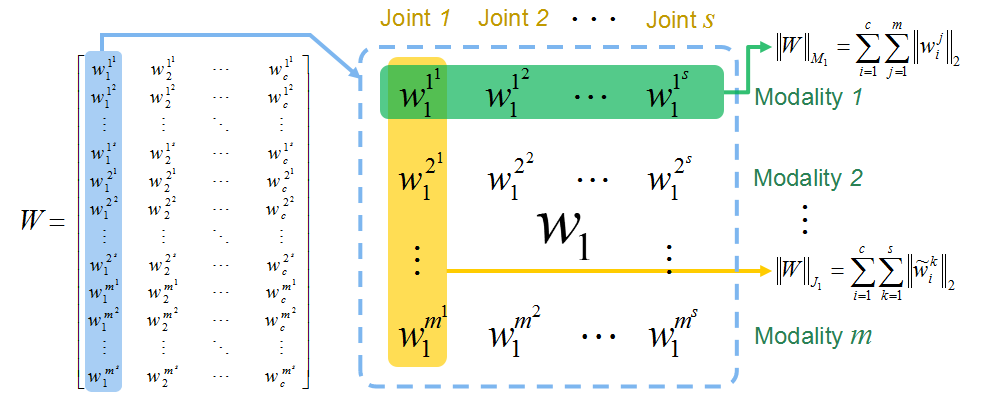}
\caption{
Illustration of the structured sparsity-inducing norms introduced in our FABL method.
Given the
parameter matrix $\mathbf{W}  = [\mathbf{w}_1, \mathbf{w}_2, \dots, \mathbf{w}_c] \in\Re^{d\times c}$,
we arrange each column vector $\mathbf{w}_i$ of the $i$-th action category into a matrix,
where rows represent modalities and columns denote skeletal joints.
We model the interrelationships of the feature modalities using the $M_1$-norm regularization term,
and the interrelationships of the skeletal joints using the $J_1$-norm regularization to model the representative joints.
}\label{eq:W}
\end{figure}

\subsection{Learning of Discriminative Body Parts}\label{sec:sub:PartLearning}

For specific behaviors,
a small set of body parts (represented as joints in human skeletal data)
are more discriminative than others.
For example,
in the behavior of hand waving as depicted in Fig. \ref{fig:sub:FASJL},
the forehand and hand joints are more discriminative.
Such discriminative human skeletal joints are typically not shared by all behavior categories
(i.e. the joints to recognize waving and kicking are substantially different).
To learn discriminative body parts,
we introduce a new joint-based group $\ell_1$-norm (named $J_1$-norm)
as a regularizer of the problem in Eq. (\ref{eq:objective_F_norm}).
The $J_1$-norm is mathematically defined as $\|\mathbf{W}\|_{J_1} \!=\! \sum_{i=1}^{c}{\sum_{k=1}^{s}{\|\mathbf{\tilde{w}}_i^k\|_2}}$,
where
$\mathbf{\tilde{w}}_i^k\in\Re^{d_k}$ denotes the weights
of the $k$-th human skeletal joint with respect to the $i$-th behavior category for all feature modalities,
which is expressed as $\mathbf{\tilde{w}}_i^k=\left[\mathbf{w}_i^{1^k};\mathbf{w}_i^{2^k};\cdots;\mathbf{w}_i^{m^k}\right]$,
and $\sum_{k=1}^{s}{d_k}=d$.
Then, we can rewrite the objective function as:
\begin{equation}\label{eq:objective_J1_norm}
\mathop{\text{min}}\limits_{\mathbf{W}}{\|\mathbf{X}^\top\mathbf{W}+\mathbf{1}_n\mathbf{b}^\top-\mathbf{Y}\|_F^2} + \gamma \|\mathbf{W}\|_{J_1}.
\end{equation}
where $\gamma$ is a trade-off hyperparameter.

The $J_1$-norm applies the $\ell_2$-norm within each skeletal joint
and the $\ell_1$-norm between the joints,
which enforces sparsity among different joints.
For example, if the skeletal features obtained from a human skeleton joint are not discriminative for a specific behavior category,
the objective in Eq. (\ref{eq:objective_J1_norm})
will assign zeros (in the ideal case, usually very small values) to them for this behavior category;
otherwise, their weights have large values.
As shown in Fig. \ref{eq:W},
the $J_1$-norm regularization term captures the interrelationship among body parts,
and estimates the importance of each body part to identify certain human behaviors.


\subsection{Learning of Multimodal Skeletal Features}

When heterogeneous multimodal features are available,
it is well accepted that different types of skeletal features show varying performance on recognizing different behaviors \cite{han2016space}.
That is,
the features from a specific modality can be more or less
discriminative for recognizing specific human behaviors.
For example, comparing to pose features,
motion features are generally less helpful to identify a still human behavior such as sitting.
To integrate multiple feature modalities
and model their interrelationships,
we introduce another group $\ell_1$-norm ($M_1$-norm)
as a new regularizer in Eq. (\ref{eq:objective_J1_norm}),
which is defined as $\|\mathbf{W}\|_{M_1} \!=\! \sum_{i=1}^{c}{\sum_{j=1}^{m}{\|\mathbf{w}_i^j\|_2}}$.
Then, incorporating
both multi-feature and multi-joint group sparsity-inducing norms,
the final objective function becomes:
\begin{equation}\label{eq:objective_final}
\mathop{\text{min}}\limits_{\mathbf{W}}{\|\mathbf{X}^\top\mathbf{W}+\mathbf{1}_n\mathbf{b}^\top-\mathbf{Y}\|_F^2} + \gamma_1 \|\mathbf{W}\|_{M_1} + \gamma_2 \|\mathbf{W}\|_{J_1}.
\end{equation}
where $\gamma_1$ and $\gamma_2$ are trade-off hyperparameters.

The $M_1$-norm uses the $\ell_2$-norm within each feature modality and the $\ell_1$-norm
between these modalities, which enforces the sparsity of these modalities.
For example, if a modality is not discriminative enough to recognize a certain behavior category,
the objective in Eq. (\ref{eq:objective_final})
will assign zeros (in the ideal case, usually very small values) to the features within this modality with respect to the behavior category;
otherwise, their weights are large.
As demonstrated in Fig \ref{eq:W}.,
the proposed $M_1$-norm regularization term captures the interrelationship between feature modalities and estimates their importance to recognize certain behaviors.

\subsection{Human Behavior Understanding}

After solving the optimization problem in  Eq. (\ref{eq:objective_final}) during the training phase
(solution is detailed in Section \ref{sec:Optimization}),
we can obtain the optimal weight matrix $\mathbf{W}^* = [\mathbf{w}_1^*, \mathbf{w}_2^*, \dots, \mathbf{w}_c^*] \in\Re^{d\times c}$.
Then, in the testing phase,
given a new multisensory instance $\mathbf{x}\in\Re^{d}$,
its behavior category $y(\mathbf{x})$ is decided by:
\begin{eqnarray}\label{eq:testingEq}
y(\mathbf{x}) = \mathop{\text{argmax}}\limits_{i}{\mathbf{x}^\top\mathbf{w}_i^*+{b}_i},\; i=1,2,...,c.
\end{eqnarray}

An advantage of our formulation utilizing the regression-like objective function
is that classification is integrated with feature learning;
thus, we do not require additional classifiers (e.g., SVMs).
This significantly improves processing efficiency,
resulting in high-speed recognition of human behaviors that can benefit real-time human-centered robotics applications.

\section{Optimization Algorithm}\label{sec:Optimization}
Since the objective in Eq. (\ref{eq:objective_final}) comprises two non-smooth
regularization terms: the $M_1$-norm and $J_1$-norm, it is difficult to solve in general.
To this end, we implement a new iterative algorithm to solve the optimization problem in Eq. (\ref{eq:objective_final})
with non-smooth regularization terms.
The proposed optimization solver has a theoretical guarantee to find the optimal solution.

To learn the value of the weight matrix $\mathbf{W}$,
we compute the derivative of the objective with respect to $\mathbf{w}_i$ $(1\leq i\leq c)$ and set
it to zero vector. Then, we obtain
\begin{eqnarray}
\mathbf{X}\mathbf{X}^\top\mathbf{w}_i - \mathbf{X}(\mathbf{y}_i-\mathbf{b}_i) + \gamma_1\mathbf{D}^i\mathbf{w}_i + \gamma_2\mathbf{\tilde{D}}^i\mathbf{w}_i = \textbf{0},
\end{eqnarray}
where $\mathbf{D}^i(1\leq i\leq c)$ is a block diagonal matrix with the $j$-th diagonal block
as $\frac{1}{2\|\mathbf{w}_i^j\|_2}\mathbf{I}_j$,
$\mathbf{w}_i^j$ is the $j$-th segment of $\mathbf{w}_i$ consisting of the weights of the $j$-th feature,
$\mathbf{\tilde{D}}^i$ is a diagonal matrix with the $k$-th
diagonal block as $\frac{1}{2\|\mathbf{\tilde{w}}_i^k\|_2}\mathbf{I}_k$, 
$\mathbf{\tilde{w}}_i^k$ is the $k$-th segment
of $\mathbf{w}_i$ including the weights of skeletal features calculated from the $k$-th skeleton joint,
and $\mathbf{I}_j$ is the identity matrix of size $d_j$.
Thus we have
\begin{eqnarray}
\mathbf{w}_i = (\mathbf{X}\mathbf{X}^\top + \gamma_1\mathbf{D}^i + \gamma_2\mathbf{\tilde{D}}^i)^{-1}\mathbf{X}(\mathbf{y}_i-\mathbf{b}_i).
\end{eqnarray}
Both $\mathbf{D}^i$ and $\mathbf{\tilde{D}}^i$ are dependent on $\mathbf{W}$ and
thus also unknown variables. An iterative algorithm is implemented to solve this problem, which
is described in Algorithm \ref{alg:1}.

Before analyzing convergence of Algorithm \ref{alg:1}, we describe a lemma from \cite{nie2010efficient} as follows.
\begin{lemma}\label{lemma1}
Given vectors $\mathbf{a}$ and $\mathbf{b}$, the following equation holds
\begin{equation}
\|\mathbf{a}\|_2-\dfrac{\|\mathbf{a}\|_2^2}{2\|\mathbf{b}\|_2}\leq\|\mathbf{b}\|_2-\dfrac{\|\mathbf{b}\|_2^2}{2\|\mathbf{b}\|_2}
\end{equation}
\end{lemma}

\begin{theorem}\label{thm1}
Algorithm \ref{alg:1} converges to the optimal solution to the optimization problem in Eq. (\ref{eq:objective_final}).
\end{theorem}
\begin{proof}
According to Step 3 of Algorithm \ref{alg:1}, we know
\begin{eqnarray}\label{eq:proof1}
&&\mathbf{W}(t+1) = \mathop{\text{argmin}}\limits_{\mathbf{W}}{\|\mathbf{X}^\top\mathbf{W}+\mathbf{1}_n\mathbf{b}^\top-\mathbf{Y}\|_F^2}\\
&&+ \gamma_1\sum_{i=1}^{c}{\mathbf{w}^\top_i\mathbf{D}^i(t+1)\mathbf{w}_i}
  + \gamma_2\sum_{i=1}^{c}{\mathbf{\tilde{w}}_i^\top\mathbf{\tilde{D}}^i(t+1)\mathbf{\tilde{w}}_i}.\nonumber
\end{eqnarray}
Then, we can derive that
\begin{eqnarray}\label{eq:proof2}
&& \mathcal{J}(t+1) + \gamma_1\sum_{i=1}^{c}{\mathbf{w}^\top_i(t+1)\mathbf{D}^i(t+1)\mathbf{w}_i(t+1)}\nonumber\\
&& +\gamma_2\sum_{i=1}^{c}{\mathbf{\tilde{w}}_i^\top(t+1)\mathbf{\tilde{D}}^i(t+1)\mathbf{\tilde{w}}_i(t+1)} \nonumber\\
&\leq & \mathcal{J}(t) +\gamma_1\sum_{i=1}^{c}{\mathbf{w}^\top_i(t)\mathbf{D}^i(t+1)\mathbf{w}_i(t)}\nonumber\\
&& + \gamma_2\sum_{i=1}^{c}{\mathbf{\tilde{w}}_i^\top(t)\mathbf{\tilde{D}}^i(t+1)\mathbf{\tilde{w}}_i(t)},
\end{eqnarray}
where $\mathcal{J}(t)=\|\mathbf{X}^\top\mathbf{W}(t)+\mathbf{1}_n\mathbf{b}^\top-\mathbf{Y}\|_F^2$.

After substituting the definition of $\mathbf{D}^i$ and $\mathbf{\tilde{D}}^i$, we obtain
\begin{eqnarray}\label{eq:proof3}
&&\mathcal{J}(t+1)+\gamma_1\sum_{i=1}^c\sum_{j=1}^m\dfrac{\|\mathbf{w}_i^j(t+1)\|_2^2}{2\|\mathbf{w}_i^j(t)\|_2}\nonumber\\
&&+\gamma_2\sum_{i=1}^c\sum_{k=1}^s\dfrac{\|\mathbf{\tilde{w}}_i^k(t+1)\|_2^2}{2\|\mathbf{\tilde{w}}_i^k(t)\|_2}\nonumber\\
&\leq & \mathcal{J}(t)+\gamma_1\sum_{i=1}^c\sum_{j=1}^m\dfrac{\|\mathbf{w}_i^j(t)\|_2^2}{2\|\mathbf{w}_i^j(t)\|_2}\nonumber\\
&&+\gamma_2\sum_{i=1}^c\sum_{k=1}^s\dfrac{\|\mathbf{\tilde{w}}_i^k(t)\|_2^2}{2\|\mathbf{\tilde{w}}_i^k(t)\|_2}.
\end{eqnarray}
From Lemma \ref{lemma1}, we can derive
\begin{eqnarray}\label{eq:proof4}
&&\sum_{j=1}^m{\|\mathbf{w}_i^j(t+1)\|_2} - \sum_{j=1}^m{\dfrac{\|\mathbf{w}_i^j(t+1)\|_2^2}{2\|\mathbf{w}_i^j(t)\|_2}}\leq\nonumber\\
&&\sum_{j=1}^m{\|\mathbf{w}_i^j(t)\|_2} - \sum_{j=1}^m{\dfrac{\|\mathbf{w}_i^j(t)\|_2^2}{2\|\mathbf{w}_i^j(t)\|_2}},
\end{eqnarray}
and
\begin{eqnarray}\label{eq:proof5}
&&\sum_{k=1}^s{\|\mathbf{\tilde{w}}_i^k(t+1)\|_2} - \sum_{k=1}^s{\dfrac{\|\mathbf{\tilde{w}}_i^k(t+1)\|_2^2}{2\|\mathbf{\tilde{w}}_i^k(t)\|_2}}\leq\nonumber\\
&&\sum_{k=1}^s{\|\mathbf{\tilde{w}}_i^k(t)\|_2} - \sum_{k=1}^s{\dfrac{\|\mathbf{\tilde{w}}_i^k(t)\|_2^2}{2\|\mathbf{\tilde{w}}_i^k(t)\|_2}}.
\end{eqnarray}
Adding Eqs. (\ref{eq:proof3})-(\ref{eq:proof5}) on both sides, we obtain
\begin{eqnarray}\label{eq:proof6}
&&\mathcal{J}(t+1)+\gamma_1\sum_{i=1}^c\sum_{j=1}^m{\|\mathbf{w}_i^j(t+1)\|_2}\\
&&+\gamma_2\sum_{i=1}^c\sum_{k=1}^s{\|\mathbf{\tilde{w}}_i^k(t+1)\|_2}\nonumber\\
&\leq & \mathcal{J}(t)+\gamma_1\sum_{i=1}^c\sum_{j=1}^m{\|\mathbf{w}_i^j(t)\|_2}
+\gamma_2\sum_{i=1}^c\sum_{k=1}^s{\|\mathbf{\tilde{w}}_i^k(t)\|_2}.\nonumber
\end{eqnarray}
Therefore, Algorithm \ref{alg:1} decreases the objective value in each iteration.
Since the optimization problem defined in Eq. (\ref{eq:objective_final}) is convex,
and the objective is lower-bounded by zero due to the definition of matrix and vector
norms, thus the algorithm converges to the optimum.
\end{proof}

\begin{algorithm}[t]
    \SetKwInOut{Input}{Input}
    \SetKwInOut{Output}{Output}

    \Input{$\mathbf{X}=\left[\mathbf{x}_1,\cdots,\mathbf{x}_n\right]\in\Re^{d\times n}$
    			and  $\mathbf{Y}=\left[\mathbf{y}_1,\cdots,								\mathbf{y}_n\right]^\top\in\Re^{n\times c}$ }
    Let $t=1$. Initialize $\mathbf{W}(t)$ by solving $\mathop{\text{min}}\limits_{\mathbf{W}}{\|\mathbf{X}^\top\mathbf{W}+\mathbf{1}_n\mathbf{b}^\top-\mathbf{Y}\|_F^2}$.

    \While{not converge}{
    Calculate the block diagonal matrix $\mathbf{D}^i(t+1)(1\leq i\leq c)$,
    where the $j$-th diagonal block of $\mathbf{D}^i(t+1)$ is
    $\frac{1}{2\|\mathbf{w}_i^j(t)\|_2}\mathbf{I}_j$.\newline
    Calculate the block diagonal matrix $\mathbf{\tilde{D}}^i(t+1)(1\leq i\leq c)$,
    where the $k$-th diagonal block of $\mathbf{\tilde{D}}^i(t+1)$ is
    $\frac{1}{2\|\mathbf{\tilde{w}}_i^k(t)\|_2}\mathbf{I}_k$.

    For each $\mathbf{w}_i(1\leq i\leq c)$, $\mathbf{w}_i(t+1) = (\mathbf{X}\mathbf{X}^\top + \gamma_1\mathbf{D}^i(t+1) + \gamma_2\mathbf{\tilde{D}}^i(t+1))^{-1}\mathbf{X}(\mathbf{y}_i-\mathbf{b}_i)$.

    $t = t + 1$.
    }
    \Output{$\mathbf{W} = \mathbf{W}(t)\in\Re^{d\times c}$}
 \caption{An iterative algorithm to solve the problem in
 Eq. (\ref{eq:objective_final})}\label{alg:1}
\end{algorithm}


\section{Experiments} \label{sec:Experiments}
To quantitatively assess the performance of the proposed FABL method,
we conduct experiments using public benchmark datasets.
Furthermore, to evaluate the benefits of our FABL method in real-world robotics applications,
we deploy FABL on a Baxter robot to perform online, real-time behavior recognition for human-robot interaction.

\subsection{Implementation}
Our FABL approach is implemented using a combination of Matlab and C++ on a Linux machine with an i7 3.4GHz CPU and 16GB memory.
The Matlab code is used to validate our approach on two public datasets:
MSR Action3D Dataset \cite{li2010action} and
Cornell Activity Dataset \cite{Sung_ICRA12},
while the C++ program is employed for validation on a Baxter robot in a real-world {``serving drinks"} task.

We intentionally designed and applied four simple skeletal features to emphasize the performance gain resulted from our FABL method instead of sophisticated features.
These simple skeletal features include:
(1) {spatial} joint displacement
that is the 3D coordinate difference of each body part with respect to the torso:
$(\Delta x, \Delta y, \Delta z) = (x, y, z) - (x^c, y^c, z^c)$,
where $(x,y,z)$ represents the coordinates of each skeletal joint,
and $(x^c,y^c,z^c)$ denotes the coordinates of the center torso joint in skeletal data,
(2) {temporal} joint displacement, which is defined as the temporal location difference of the same body joint in the current frame with respect to the previous frame:
$(\dot{x}, \dot{y}, \dot{z}) = (x_{t}, y_t, z_t) - (x_{t-1}, y_{t-1}, z_{t-1})$,
where $(x_{t}, y_t, z_t)$ is the joint location at time $t$,
(3) long-term temporal joint displacement,
defined as the temporal 3D location difference between the current frame and the initial frame:
$(\dot{\tilde{x}}, \dot{\tilde{y}}, \dot{\tilde{z}}) = (x_{t}, y_t, z_t) - (x_{0}, y_{0}, z_{0})$,
where $(x_{0}, y_{0}, z_{0})$ is the coordinates of a joint in the initial frame,
and (4) spatial joint distance,
which is defined as the geometrical distance of a joint to the torso center joint: $d = \|(x, y, z) - (x^c, y^c, z^c)\|_2$.
Then, we compute a histogram of each feature type to build a vector that is used as a feature modality in our experiment.


\subsection{Results on MSR Action3D Dataset}

We evaluate the performance of the proposed approach to recognize human behaviors when interacting with structured-light cameras,
using the MSR Action3D benchmark dataset
\cite{li2010action}.
This dataset contains 20 categories of human actions performed by 7 subjects for three times. The skeleton sequence of ``high arm waving" is shown in Fig. \ref{fig:MSR_dataset}.

\begin{figure}[htb]
\centering
\includegraphics[width=3.4in]{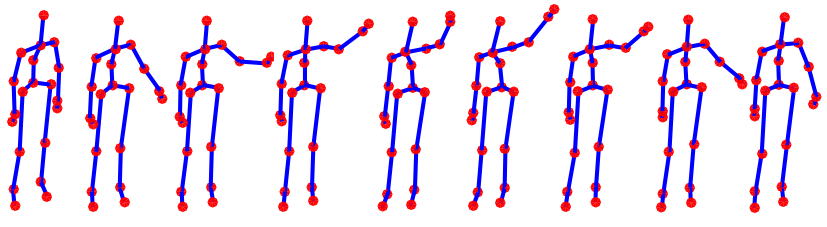}
\caption{The MSR Action3D dataset is utilized in the experiment to evaluate the proposed FABL approach, which contains 20 activities recorded using Kinect, which are (M1) high arm wave, (M2) horizontal arm wave, (M3) hammer, (M4) hand catch, (M5) forward punch, (M6) high throw, (M7) draw x, (M8) draw tick, (M9) draw circle, (M10) hand clap, (M11) two hand wave, (M12) side boxing, (M13) bend, (M14) forward kick, (M15) side kick, (M16) jogging, (M17) tennis swing, (M18) tennis serve, (M19) golf swing, and (M20) pick up \& throw. This figure shows a sample skeleton sequence of the action (M1) high arm waving in the dataset}\label{fig:MSR_dataset}
\end{figure}

We evaluate the recognition performance using a challenging subject-wise setting.
That is, the training dataset does not contain any data instances from the subjects who participate in testing.
When combined both structured sparsity-inducting norms to perform simultaneous feature and skeletal joint learning,
our FABL method obtains an accuracy of 91.67\%,
The confusion matrix obtained by our method is shown in Fig. \ref{fig:MSR_confusionMatrix},
which demonstrates our FABL approach is able to well recognize most of the behaviors.
The actions that are not well identified is (M4) hand-catch, and (M7) draw-x that is always misclassified as the action of (M8) draw tick or (M9) draw circle, which have similar, small motions.

\begin{figure}[htb]
  \subfigure[MSR Action3D dataset]{
    \label{fig:MSR_confusionMatrix} 
    \begin{minipage}[b]{0.23\textwidth}
      \centering
        \includegraphics[height=1.55in]{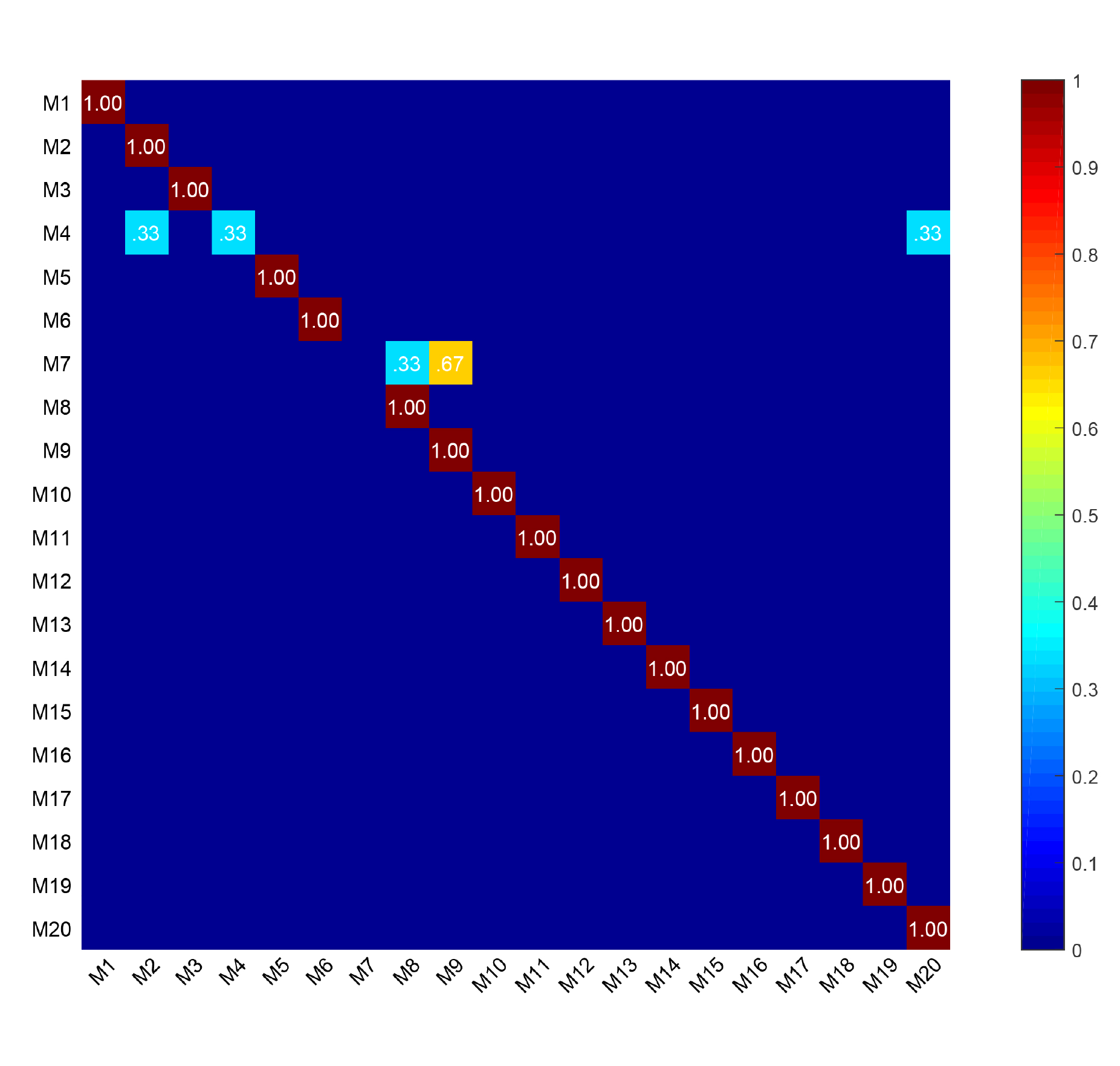}
    \end{minipage}}
  \subfigure[CAD-60 dataset]{
     \label{fig:CAD60_confusionMatrix}
    \begin{minipage}[b]{0.22\textwidth}
      \centering
        \includegraphics[height=1.55in]{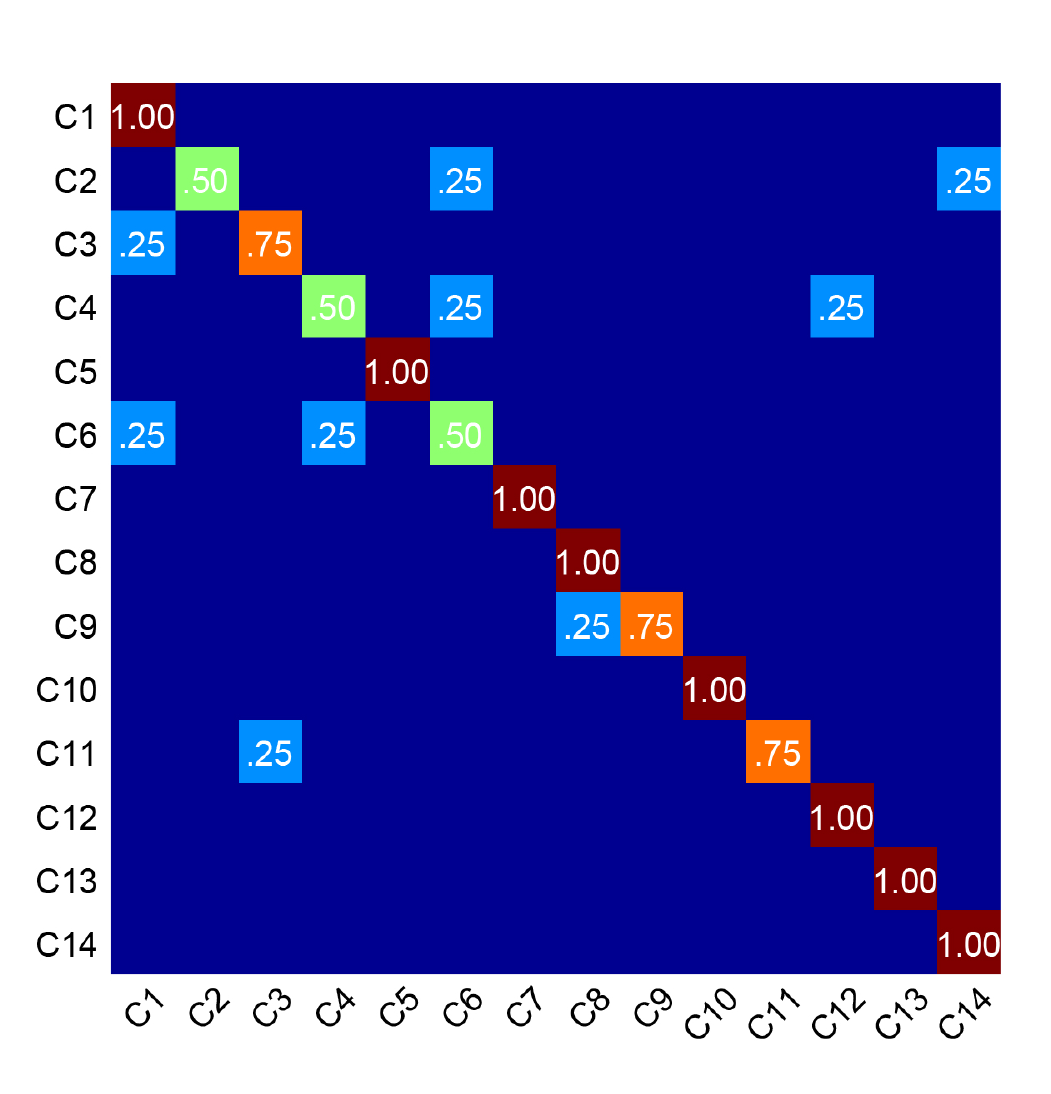}
    \end{minipage}}
  \caption{
  Confusion matrices obtained by our FABL method over the MSR Action3D and CAD-60 dataset datasets. The behavior category labels \emph{M1-M16} and \emph{C1-C14} are described in Fig. \ref{fig:MSR_dataset} and Fig. \ref{fig:CAD_dataset}, respectively.
  }
  \label{datasets_confusionMatrix} 
\end{figure}

We compare with two baseline methods including feature-learning-only ($\gamma_2=0$) and body-part-learning-only ($\gamma_1=0$).
As presented in Table \ref{tab:MSR_comparison},
the feature-learning-only method obtains an average recognition accuracy of 85.00\%,
while the body-part-learning-only obtains an average accuracy of 86.67\%.
This indicates that FABL outperforms baseline approaches using a single norm for regularization.
In addition, we compare our FABL method with previous activity recognition techniques based on skeleton features.
FABL achieves promising recognition accuracy
(with the high-speed performance) on the MSR Action3D dataset.

\begin{table}[htbp]
\centering
\caption{Comparison of average accuracy with previous skeleton-based methods on
the MSR Action3D dataset}
\label{tab:MSR_comparison}
\tabcolsep=0.11cm
\begin{tabular}{|c|c|c|}
\hline
Reference & Method & Accuracy\\
\hline\hline
Ofli et al. \cite{ofli2014sequence} & Sequence
of Most Informative Joints & 41.18\% \\
\hline
Wang \emph{et al.} \cite{wang2012mining} & Dynamic Temporal Warping & 54.0\% \\
\hline
Ellis et al. \cite{ellis2013exploring} & Joints Distance + Key Poses & 65.7\% \\
\hline
Li \emph{et al.} \cite{li2010action} & Action Graph & 74.7\% \\
\hline
Xia et al. \cite{xia2012view} & HOJ3D & 78\% \\
\hline
Yang and Tian \cite{yang2014effective} & EigenJoints & 83.3\% \\
\hline
Wang \emph{et al.} \cite{wang2014learning} & Actionlet Ensemble & 88.2\% \\
\hline
Ben Amor \emph{et al.} \cite{amor2016action} & Skeleton Trajectories & 89\% \\
\hline\hline
 & Feature Learning Only & 85.00\%\\
Our Methods & Body-Part Learning Only & 86.67\%\\
 & \textbf{FABL} & \textbf{91.67\%}\\
\hline
\end{tabular}
\end{table}

\subsection{Results on Cornell Activity Dataset}
The Cornell Activity Dataset 60 (CAD-60)
 \cite{Sung_ICRA12}
is a widely applied benchmark for human activity recognition in robotics applications.
This dataset includes color-depth and skeleton information of twelve daily activities
as well as two motions {``still"} and {``random" } recorded by a Kinect sensor in various environments,
including office, kitchen, bedroom, bathroom, and living room. Each activity is performed by four subjects with two males and two females (one subject is left-handed). The skeleton data in each frame contains 15 joints, as shown in Figure \ref{fig:CAD_dataset}.
We evaluate FABL's performance in a subject-wise cross-validation setup \cite{ni2012order}, where actions performed by new subjects are used for testing.

\begin{figure}[htb]
\centering
\includegraphics[width=3.35in]{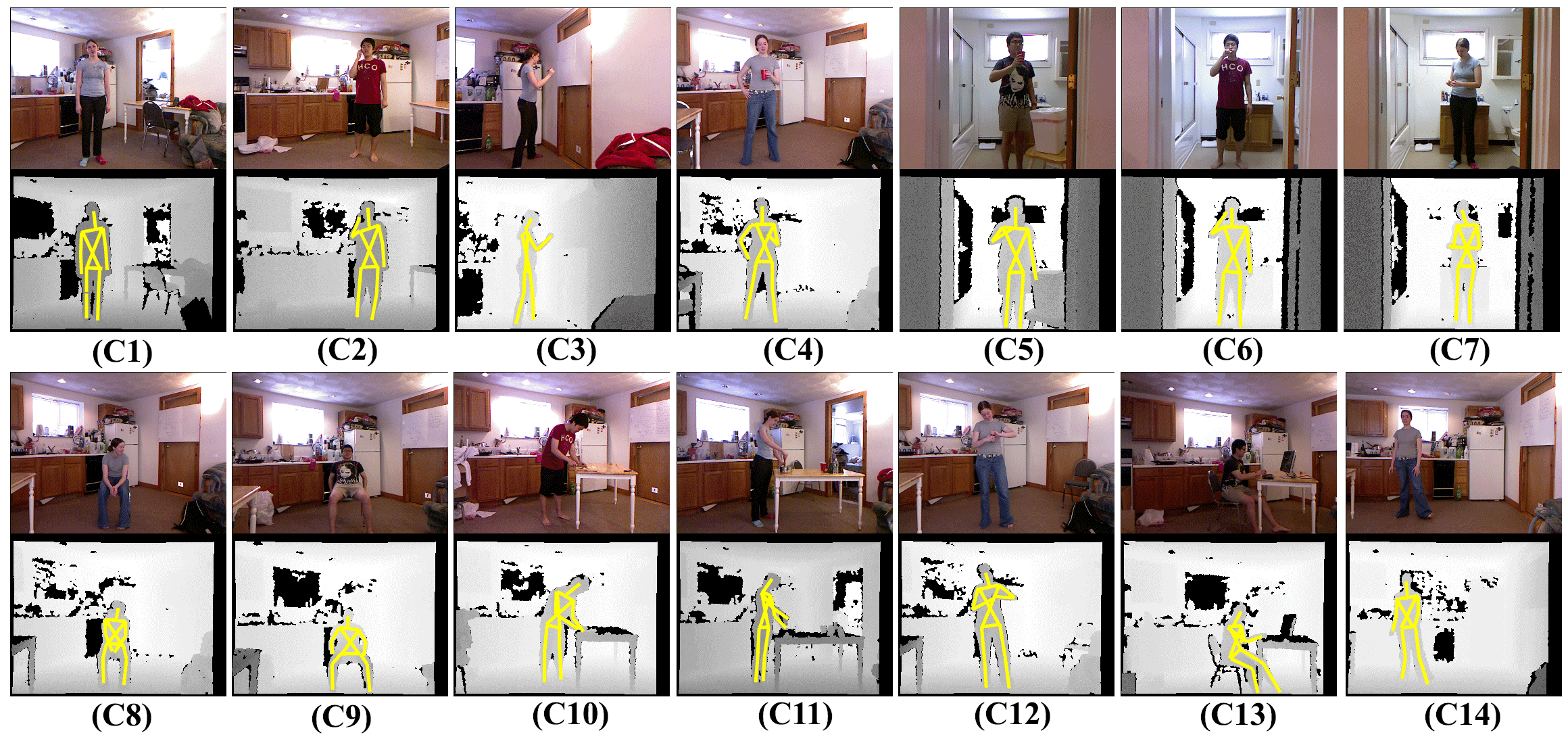}
\caption{The CAD-60 dataset contains 14 behaviors, including (C1) standing still, (C2) talking on the phone, (C3) writing on whiteboard, (C4) drinking water, (C5) rinsing mouth with water, (C6) brushing teeth, (C7) wearing contact lenses, (C8) talking on couch, (C9) relaxing on couch, (C10) cooking (chopping), (C11) cooking (stirring), (C12) opening pill container, (C13) working on computer, (C14) random.
RGB images are depicted in the top row, and the depth images with the human skeleton in yellow are shown in the bottom row.
}\label{fig:CAD_dataset}
\end{figure}

As demonstrated in Table \ref{tab:CAD60_comparison},
the FABL method using both regularization terms obtain an average accuracy of 83.93\%,
and its detailed confusion matrix is graphically presented in Fig. \ref{fig:CAD60_confusionMatrix},
which generally indicates that most of the activities can be well classified by our approach.

\newcommand{\tabincell}[2]{\begin{tabular}{@{}#1@{}}#2\end{tabular}}
\begin{table}[htb]
\centering
\caption{Comparison of average recognition accuracy with previous skeleton-based methods on the CAD-60 dataset}
\label{tab:CAD60_comparison}
\tabcolsep=0.15cm
\begin{tabular}{|c|c|c|}
\hline
Reference & Method & Accuracy\\
\hline\hline
Ni \emph{et al.} \cite{ni2012order} & Order-Preserving Sparse Coding  & 65.32\% \\
\hline
\tabincell{c}{Piyathilaka and\\Kodagoda \cite{piyathilaka2013gaussian}} & Hidden Markov Model & 78.38\% \\
\hline
Wang \emph{et al.} \cite{wang2014learning} & Skeleton-based Actionlet Ensemble  & 74.70\% \\
\hline
Zhang and Tian \cite{zhang2012rgb} & Bag of Features & 80.77\% \\

\hline\hline
 & Feature Learning Only & 78.57\%\\
Our Methods & Body-Part Learning Only & 79.46\%\\
 &  \textbf{FABL} & \textbf{83.93\%}\\
\hline
\end{tabular}
\end{table}


We implemented two baseline techniques under the same formulation. First, we
set $\gamma_2=0$ to evaluate the performance of the feature learning scheme,
and obtain an accuracy of 78.57\%.
Then, $\gamma_1$ is set to zero to evaluate the performance of the body-part learning scheme,
and we obtain an average accuracy of 79.46\%.
It is observed that both baseline methods perform worse than the full FABL approach using both regularization terms.
Moreover, we implemented a third baseline method with no regularization terms,
which obtains an accuracy of 76.79\% and performs worse than the methods with the regularization terms.
In addition, we compare our FABL method with previous state-of-the-art skeleton-based techniques for activity recognition,
as reported in Table \ref{tab:CAD60_comparison},
which shows our FABL method outperforms these skeleton-based techniques over the CAD-60 dataset.

\subsection{Behavior Recognition for Human-Robot Interaction}
Besides using public benchmark datasets to evaluate and compare our FABL method's accuracy,
we also implemented and deployed the method on a physical robot to validate its performance in real-world robotics applications.
The robot employed in this experiment is a Baxter robot, as shown in Fig. \ref{fig:baxter},
which uses a structured-light sensor for onboard 3D perception
and the same workstation (Intel i7 3.4GHz CPU and 16GB memory) for onboard control and data processing.

\begin{figure}[htb]
  \hspace{-0.05in}
  \subfigure[Baxter performing ``serving drinks'']{
     \label{fig:baxter}
    \begin{minipage}[b]{0.265\textwidth}
      \centering
        \includegraphics[height=1.35in]{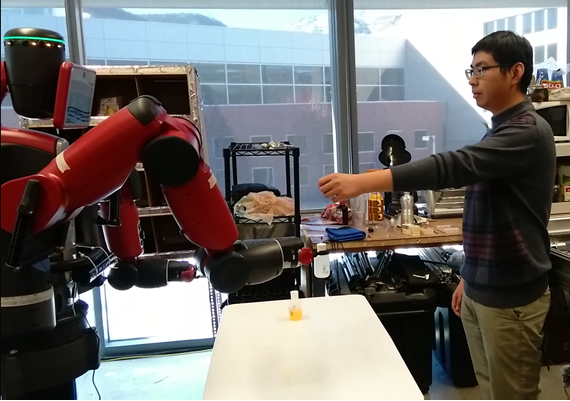}
    \end{minipage}}
  \subfigure[Confusion matrix]{
    \label{fig:baxter_confusionMatrix} 
    \begin{minipage}[b]{0.175\textwidth}
      \centering
        \includegraphics[height=1.4in]{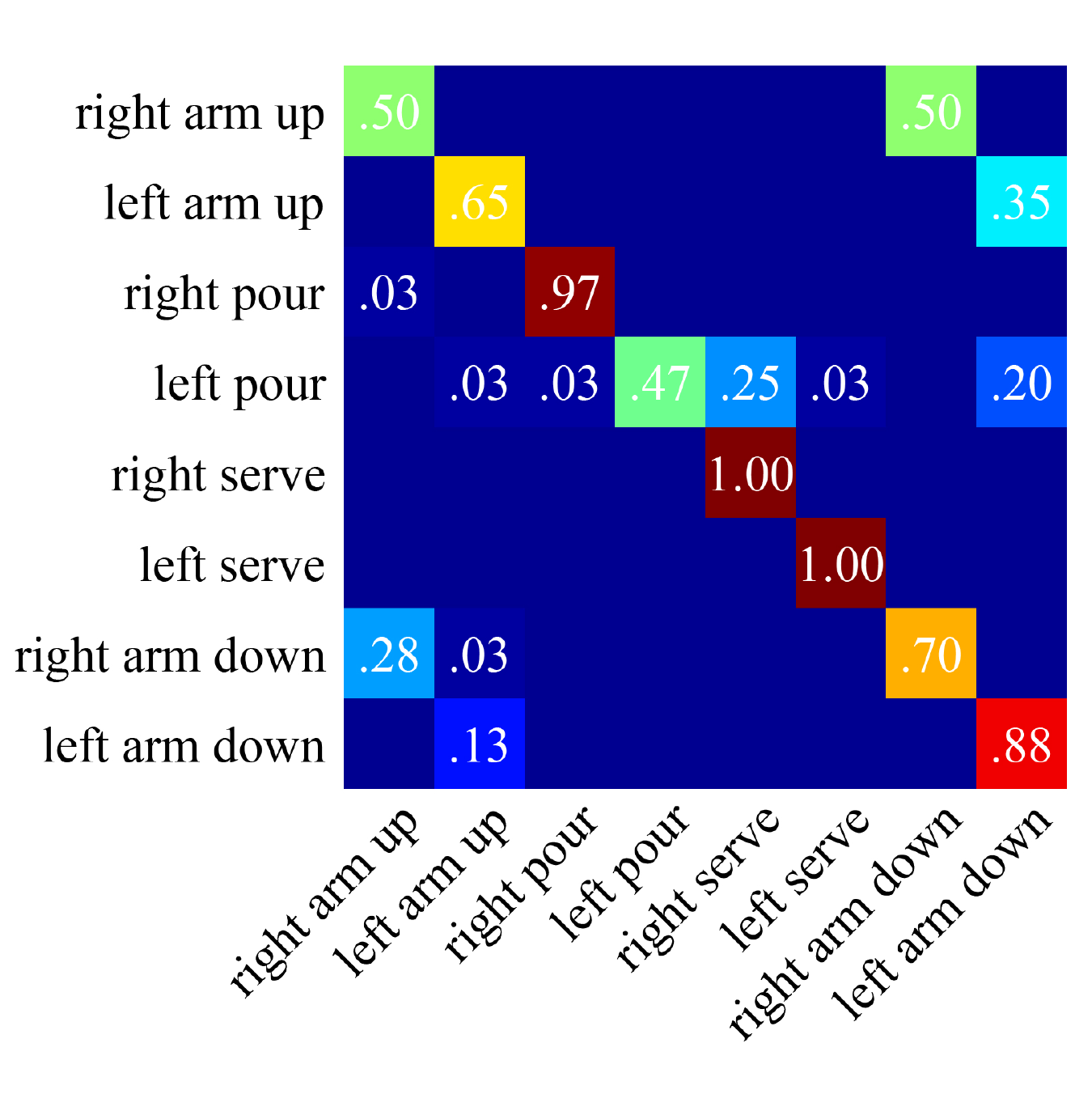}
    \end{minipage}}
  \caption{We evaluate our FABL approach using a Baxter robot to recognize behaviors for real-time human-robot interaction. The tasks focus on the robot-assisted living application such as  ``serving drinks'' as shown in Fig. \ref{fig:baxter}.
  The confusion matrix is illustrated in Fig. \ref{fig:baxter_confusionMatrix}.
  }
  \label{fig:baxter_experiment} 
\end{figure}

In this experiment, the task focuses on the robot-assisted living application,
where the Baxter robot needs to recognize the activities of a subject
and perform a collection of predefined robot actions, such as ``serving drinks,'' as demonstrated in Fig. \ref{fig:baxter},
in response to the subject's activity.
We define six robot actions, including
fetching a drinking bottle with one hand,
fetching an empty cup with the other hand,
pouring the drinks into the cup,
putting back the bottle,
serving the drinking cup to the subject,
and finally putting back the cup.
Each robot action is triggered
by a specific command gesture performed
by a subject in front of the robot,
which must be recognized by the Baxter robot.
The skeleton data is captured onboard and in real time using ROS and the OpenNI package.

Eight human behavior categories are defined and used to interact with the robot,
including lifting up left/right arms, pouring with left/right hands,
serving with left/right hands, and putting down left/right arms.
We specifically distinguish between left side and right side,
because this is critical to take into account human preference in practical, real-world scenarios.
Two human subjects having different body scales and motion patterns are involved in this experiment.
Each subject performs each of the eight behaviors 20 times. Actions by one subject were used for training, while other subject's actions were used for testing.
Ground truth is manually recorded and used to compare with recognition results obtained by the robot for quantitative evaluation.
After extracting multimodal features from training instances,
our method computes the optimal weight matrix by Algorithm \ref{alg:1} using the training data.
Then, the learned FABL approach is deployed on the robot for online, onboard behavior recognition to enable real-time human-robot interaction.

Similar to the experiments using public datasets,
we also quantitatively assess FABL's performance and compare with baseline and existing skeleton-based techniques.
The average accuracy obtained by the complete FABL method is 77.19\%
with both regularization terms.
The confusion matrix obtained by our FABL approach is demonstrated in Fig. \ref{fig:baxter_confusionMatrix}.
For comparison, the baseline technique based only on feature learning ($\gamma_2=0$)
obtains an accuracy of 76.56\%,
while the baseline based only on body part learning ($\gamma_1=0$)
obtains an average recognition accuracy of 76.25\%.
In addition, we compare our FABL method with several previous skeleton-based recognition techniques and present the results in Table \ref{tab:Compare_baxter}.
We can observe that the FABL method is able to obtain better performance over baseline and used previous methods.
Since only one subject's actions were used to train the FABL model, the recognition accuracy was not as significant as that using public benchmark datasets. More training data will improve the testing performance.

\begin{table}[tbp]
\centering\tabcolsep=0.1cm
\caption{Comparison of average recognition accuracy with previous methods for real-time human-robot interaction}\label{tab:Compare_baxter}
\begin{tabular}{|c|c|c|}
\hline
Reference & Method & Accuracy\\
\hline\hline
\cite{ellis2013exploring} & Relative Angles and Distances & 15.00\% \\\hline
\cite{wang2014learning} & Histogram of Joint Position Differences & 48.13\% \\\hline
\cite{gowayyed2013histogram} & Histogram of Oriented Displacements & 51.25\% \\
\hline\hline
& Feature Learning Only & 76.56\%\\
Our Methods & Body-Part Learning Only & 76.25\%\\
 & \textbf{FABL}   &  \textbf{77.19\%}\\\hline
\end{tabular}
\end{table}

\section{Discussion}\label{sec:discussion}

\noindent\textbf{High-Speed Processing.}
Due to the capability of our FABL approach to integrate both feature learning and classification in the same formulation,
and the efficiency of our regression-like objective function,
our FABL approach is able to achieve high-speed processing.
To validate this strong advantage,
we perform additional experiments over the MSR Action3D and CAD-60 datasets using Matlab implementations without any optimization,
and utilizing the real Baxter robot using a C++ implementation.
The runtime results on all used datasets are presented in Table \ref{tab:Compare_time},
which shows our FABL approach can achieve a significantly high processing speed at the order of $10^4$ Hz.
This indicates the promise of our FABL approach to identify human behaviors in real-time robotics applications.


\begin{table}[htbp]
\centering\tabcolsep=0.1cm
\caption{Runtime Analysis Over Different Datasets}\label{tab:Compare_time}
\begin{tabular}{|l|c|c|c|}
\hline
Runtime & MSR Action3D & CAD-60 & Baxter\\
\hline\hline
Processing speed (Hz) & $2.2\times 10^{4}$ & $1.4\times 10^{4}$ & $3.3\times 10^{4}$ \\
Time per observation (s) & $4.5\times 10^{-5}$ & $7.3\times 10^{-5}$ & $3.0\times 10^{-5}$ \\
\hline
\end{tabular}
\end{table}

\noindent\textbf{Generalizability.}
FABL is a general approach that can work with different body kinematic models obtained by a variety of sensing devices and skeleton generation packages,
including the OpenNI package in ROS,
Microsoft SDKs, and MoCap systems.
Given any kinematic body model from the devices,
we can downsample the body model into 15 body parts,
and apply FABL to automatically identify the most representative parts. 
In this case, FABL can achieve cross-training \cite{Zhang15},
i.e., methods trained on a kinematic body model from one device can be directly applied to other models by a different device,
which can significantly save design labor.



\noindent\textbf{Hyperparameter Selection.}
The regularization hyperparameters $\gamma_1$ and $\gamma_2$ are utilized to control the effect of feature learning and the strength of body-part learning, respectively.
Their optimal values can be decided using cross-validation during the training process.
In general,
we observe that the values $\gamma_1 = 0.1$ and $\gamma_2=0.1$ usually result in satisfactory recognition accuracy,
which shows that both regularization terms are necessary.
When the values of hyperparameters become too large,
the performance decreases,
because the loss function that models the recognition error is more ignored.
When $\gamma_1$ and $\gamma_2$ take too small values,
the approach cannot well capture the interrelationships of feature modalities and body parts,
thus decreasing the recognition accuracy.

%

\section{Conclusion}\label{sec:Conclusion}
In this paper,
we introduce a novel FABL approach
that is able to simultaneously learn discriminative feature modalities and body parts to perform high-speed human behavior recognition.
The proposed FABL method automatically identifies discriminative feature modalities and important body parts
using two structured sparsity-inducing norms to model their interrelationships.
Our FABL approach formulates behavior recognition as a regression-like optimization problem,
which is solved by an efficient iteration algorithm that possesses a theoretical guarantee to find the optimal solution.
To evaluate the performance of the proposed FABL method,
we perform empirical studies using two public benchmark datasets
and a physical Baxter robot.
The experimental results have indicated that FABL is able to outperform  existing skeleton-based methods.
More importantly, our FABL approach achieves a high processing speed of more than $10^4$ Hz,
which can enable realistic, self-contained, intelligent robots to recognize human behaviors and interact with humans in real time.

\bibliographystyle{IEEEtran}
\bibliography{references_abbr}

\begin{thebibliography}{10}
\providecommand{\url}[1]{#1}
\csname url@rmstyle\endcsname
\providecommand{\newblock}{\relax}
\providecommand{\bibinfo}[2]{#2}
\providecommand\BIBentrySTDinterwordspacing{\spaceskip=0pt\relax}
\providecommand\BIBentryALTinterwordstretchfactor{4}
\providecommand\BIBentryALTinterwordspacing{\spaceskip=\fontdimen2\font plus
\BIBentryALTinterwordstretchfactor\fontdimen3\font minus
  \fontdimen4\font\relax}
\providecommand\BIBforeignlanguage[2]{{%
\expandafter\ifx\csname l@#1\endcsname\relax
\typeout{** WARNING: IEEEtran.bst: No hyphenation pattern has been}%
\typeout{** loaded for the language `#1'. Using the pattern for}%
\typeout{** the default language instead.}%
\else
\language=\csname l@#1\endcsname
\fi
#2}}

\bibitem{zhang2014simplex}
H.~Zhang, W.~Zhou, C.~Reardon, and L.~E. Parker, ``Simplex-based {3D}
  spatio-temporal feature description for action recognition,'' in \emph{CVPR},
  2014.

\bibitem{wang2014learning}
J.~Wang, Z.~Liu, Y.~Wu, and J.~Yuan, ``Learning actionlet ensemble for {3D}
  human action recognition,'' \emph{TPAMI}, vol.~36, no.~5, pp. 914--927, 2014.

\bibitem{shotton2011real}
J.~Shotton, A.~Fitzgibbon, M.~Cook, T.~Sharp, M.~Finocchio, R.~Moore,
  A.~Kipman, and A.~Blake, ``Real-time human pose recognition in parts from
  single depth images,'' in \emph{CVPR}, 2011.

\bibitem{han2016space}
F.~Han, B.~Reily, W.~Hoff, and H.~Zhang, ``Space-time representation of people
  based on {3D} skeletal data: A review,'' \emph{CVIU}, 2017, to appear.

\bibitem{Zhang15}
H.~Zhang and L.~E. Parker, ``Bio-inspired predictive orientation decomposition
  of skeleton trajectories for real-time human activity prediction,'' in
  \emph{ICRA}, 2015.

\bibitem{Sung_ICRA12}
J.~Sung, C.~Ponce, B.~Selman, and A.~Saxena, ``Unstructured human activity
  detection from {RGBD} images,'' in \emph{ICRA}, 2012.

\bibitem{Yang_CVPRW12}
X.~Yang and Y.~Tian, ``{EigenJoints}-based action recognition using
  {N{\"a}ive-Bayes-Nearest-Neighbor},'' in \emph{CVPRW}, 2012.

\bibitem{gowayyed2013histogram}
M.~A. Gowayyed, M.~Torki, M.~E. Hussein, and M.~El-Saban, ``Histogram of
  oriented displacements {(HOD)}: describing trajectories of human joints for
  action recognition,'' in \emph{IJCAI}, 2013.

\bibitem{yang2014effective}
X.~Yang and Y.~Tian, ``Effective {3D} action recognition using {EigenJoints},''
  \emph{JVCIR}, vol.~25, no.~1, pp. 2--11, 2014.

\bibitem{wang2012mining}
J.~Wang, Z.~Liu, Y.~Wu, and J.~Yuan, ``Mining actionlet ensemble for action
  recognition with depth cameras,'' in \emph{CVPR}, 2012.

\bibitem{ofli2014sequence}
F.~Ofli, R.~Chaudhry, G.~Kurillo, R.~Vidal, and R.~Bajcsy, ``Sequence of the
  most informative joints {(SMIJ)}: A new representation for human skeletal
  action recognition,'' \emph{JVCIR}, vol.~25, no.~1, pp. 24--38, 2014.

\bibitem{chen2013online}
X.~Chen and M.~Koskela, ``Online {RGB-D} gesture recognition with extreme
  learning machines,'' in \emph{ICMI}, 2013.

\bibitem{Wei2013}
P.~Wei, Y.~Zhao, N.~Zheng, and S.-C. Zhu, ``Modeling {4D} human-object
  interactions for event and object recognition,'' in \emph{ICCV}, 2013.

\bibitem{wei2013concurrent}
P.~Wei, N.~Zheng, Y.~Zhao, and S.-C. Zhu, ``Concurrent action detection with
  structural prediction,'' in \emph{AAAI}, 2013.

\bibitem{yu2015discriminative}
G.~Yu, Z.~Liu, and J.~Yuan, ``Discriminative orderlet mining for real-time
  recognition of human-object interaction,'' in \emph{ACCV}, 2014.

\bibitem{Masood2011}
S.~Z. Masood, C.~Ellis, A.~Nagaraja, M.~F. Tappen, J.~J.~L. Jr., and
  R.~Sukthankar, ``Measuring and reducing observational latency when
  recognizing actions,'' in \emph{ICCV}, 2011.

\bibitem{wang2013approach}
C.~Wang, Y.~Wang, and A.~L. Yuille, ``An approach to pose-based action
  recognition,'' in \emph{CVPR}, 2013.

\bibitem{xiaohan2015joint}
B.~X. Nie, C.~Xiong, and S.-C. Zhu, ``Joint action recognition and pose
  estimation from video,'' in \emph{CVPR}, 2015.

\bibitem{du2015hierarchical}
Y.~Du, W.~Wang, and L.~Wang, ``Hierarchical recurrent neural network for
  skeleton based action recognition,'' in \emph{CVPR}, 2015.

\bibitem{chaaraoui2014evolutionary}
A.~A. Chaaraoui, J.~R. Padilla-L{\'o}pez, P.~Climent-P{\'e}rez, and
  F.~Fl{\'o}rez-Revuelta, ``Evolutionary joint selection to improve human
  action recognition with {RGB-D} devices,'' \emph{ESA}, vol.~41, no.~3, pp.
  786--794, 2014.

\bibitem{reyes2011featureweighting}
M.~Reyes, G.~Dom{\'\i}nguez, and S.~Escalera, ``Feature weighting in dynamic
  timewarping for gesture recognition in depth data,'' in \emph{ICCVW}, 2011.

\bibitem{patsadu2012human}
O.~Patsadu, C.~Nukoolkit, and B.~Watanapa, ``Human gesture recognition using
  {Kinect} camera,'' in \emph{IJCCSSE}, 2012.

\bibitem{huang2014action}
D.-A. Huang and K.~M. Kitani, ``Action-reaction: Forecasting the dynamics of
  human interaction,'' in \emph{ECCV}, 2014.

\bibitem{Chaudhry2013}
R.~Chaudhry, F.~Ofli, G.~Kurillo, R.~Bajcsy, and R.~Vidal, ``Bio-inspired
  dynamic {3D} discriminative skeletal features for human action recognition,''
  in \emph{CVPR}, 2013.

\bibitem{nie2010efficient}
F.~Nie, H.~Huang, X.~Cai, and C.~H. Ding, ``Efficient and robust feature
  selection via joint $\ell_{2,1}$-norms minimization,'' in \emph{NIPS}, 2010.

\bibitem{li2010action}
W.~Li, Z.~Zhang, and Z.~Liu, ``Action recognition based on a bag of {3D}
  points,'' in \emph{CVPRW}, 2010.

\bibitem{ellis2013exploring}
C.~Ellis, S.~Z. Masood, M.~F. Tappen, J.~J. Laviola~Jr, and R.~Sukthankar,
  ``Exploring the trade-off between accuracy and observational latency in
  action recognition,'' \emph{IJCV}, vol. 101, no.~3, pp. 420--436, 2013.

\bibitem{xia2012view}
L.~Xia, C.-C. Chen, and J.~Aggarwal, ``View invariant human action recognition
  using histograms of {3D} joints,'' in \emph{CVPRW}, 2012.

\bibitem{amor2016action}
B.~Ben~Amor, J.~Su, and A.~Srivastava, ``Action recognition using
  rate-invariant analysis of skeletal shape trajectories,'' \emph{TPAMI},
  vol.~38, no.~1, pp. 1--13, 2016.

\bibitem{ni2012order}
B.~Ni, P.~Moulin, and S.~Yan, ``Order-preserving sparse coding for sequence
  classification,'' in \emph{ECCV}, 2012.

\bibitem{piyathilaka2013gaussian}
L.~Piyathilaka and S.~Kodagoda, ``Gaussian mixture based {HMM} for human daily
  activity recognition using {3D} skeleton features,'' in \emph{CIEA}, 2013.

\bibitem{zhang2012rgb}
C.~Zhang and Y.~Tian, ``{RGB-D} camera-based daily living activity
  recognition,'' \emph{CVIP}, vol.~2, no.~4, p.~12, 2012.

\end{thebibliography}

\end{document}